\documentclass[final,12pt]{colt2026} % Include author names

\title[Universal Functionals Approximation]{Universal Approximation of Continuous Functionals on Compact Subsets of $H^n$ via Linear Measurements and Scalar Nonlinearities}
\usepackage{times}
\coltauthor{
 \Name{Andrey Krylov} \Email{kryl@cs.msu.ru}\\
 \Name{Maksim Penkin} \Email{penkin97@gmail.com}\\
 \addr {Lomonosov Moscow State University, Russia}
}

\begin{document}

\maketitle

\begin{abstract}%
We study universal approximation of continuous functionals on compact subsets of products of Hilbert spaces.
We prove that any such functional can be uniformly approximated by models that first take finitely many continuous linear measurements of the inputs and then combine these measurements through continuous scalar nonlinearities.
We also extend the approximation principle to maps with values in a Banach space, yielding finite-rank approximations.
These results provide a compact-set justification for the common ``measure, apply scalar nonlinearities, then combine'' design pattern used in operator learning and imaging.
\end{abstract}

\begin{keywords}%
  universal approximation; continuous maps; operator-valued approximation; Banach spaces; compact sets; linear measurements; scalar nonlinearities; finite-rank approximation; Stone--Weierstrass; neural operators; inverse problems, artificial intelligence%
\end{keywords}

\section{Introduction}
Many modern learning tasks are really about functionals: given inputs $x=(x_1,\dots,x_n)$ in a product Hilbert space $H^n$, we want to model a continuous map $f\colon K\to\mathbb{R}$ on a compact set $K\subset H^n$. This viewpoint is natural in operator learning and scientific machine learning, where the data are functions, fields, or signals rather than fixed-length vectors \cite{kovachki_neuraloperator,li_fno,lu_deeponet}.

A practical design pattern is to separate sensing from learning. First, we take finitely many linear measurements of each function-valued input (sensor readings, basis coefficients, filtered responses, projection measurements). Second, we feed these numbers into a nonlinear model (typically a neural network) to produce the output. This measure-then-nonlinear-combine structure sits behind many operator-learning architectures \cite{lu_deeponet,li_fno,kovachki_neuraloperator} and also classical imaging pipelines that mix linear filtering/projection with pointwise nonlinearities \cite{mallat_wavelet,ronneberger_unet,adler_oktem_learned}.

This paper provides a compact-set theorem that justifies this architecture class at the level of universal approximation. On compact subsets of $H^n$ (with $H$ a Hilbert space), we show that any continuous functional $f\colon K\to\mathbb{R}$ can be uniformly approximated by finite sums of the form
\[
  \sum_{j=1}^r \zeta_j\Bigl(\sum_{i=1}^n \varphi_{ji}(x_i)\Bigr),
\]
where $\varphi_{ji}\in H^*$ are continuous linear measurements and $\zeta_j\colon\mathbb{R}\to\mathbb{R}$ are continuous scalar nonlinearities. In operator-learning language, $\{\varphi_{ji}\}$ defines the measurement (encoding) stage and $\{\zeta_j\}$ defines the nonlinear decoder/mixer.

In Appendix~\ref{app:operator-valued} we extend the same compact-set approximation principle to operator-valued outputs $f\colon K\to Y$ for general Banach spaces $Y$, and we briefly discuss applications for computer vision and AI in Appendices~\ref{app:cv}--\ref{app:ai}.

Beyond stating the approximation result, we aim to make the proof engineering-readable: compactness lets us reduce to a finite-dimensional proxy via orthogonal projection, and then classical approximation results take over. This perspective explains why a finite set of measurements can suffice on realistic (compact) data regimes, and it links PINNs, INRs, and imaging/inverse problems through the same two-stage template.

\section{Main result}

\medskip

\begin{theorem}[Approximation on a compact set]
Let $K\subset H^n$ be compact and let $f\in C(K)$. Then for every $\varepsilon>0$ there exists a function of the form
\[
  g(x_1,\dots,x_n)=\sum_{j=1}^r \zeta_j\Bigl(\sum_{i=1}^n \varphi_{ji}(x_i)\Bigr),
\]
where $\varphi_{ji}\in H^*$ and $\zeta_j\colon\mathbb{R}\to\mathbb{R}$ are continuous, such that
\[
  \sup_{(x_1,\dots,x_n)\in K}\,|f(x_1,\dots,x_n)-g(x_1,\dots,x_n)|<\varepsilon.
\]
\end{theorem}

\begin{proof}
Let $X:=H^n$, and regard $X$ as a Hilbert space with norm
\[
  \|(x_1,\dots,x_n)\|^2:=\sum_{i=1}^n\|x_i\|_H^2.
\]

\textbf{Step 1 (uniform continuity).} Since $K$ is compact and $f\in C(K)$, the function $f$ is uniformly continuous on $K$ (the Heine--Cantor theorem) \cite{rudin_principles}. Choose $\delta>0$ such that for any $x,x'\in K$, the condition $\|x-x'\|<\delta$ implies
\[
  |f(x)-f(x')|<\varepsilon/3.
\]

\textbf{Step 2 (reduction to the finite-dimensional case).} Since $K$ is compact, from any open cover one can extract a finite subcover \cite{munkres_topology}; in particular, there exists a finite $\delta/3$-covering: there are points $x^{(1)},\dots,x^{(m)}\in K$ such that
\[
  K\subset \bigcup_{k=1}^m B\bigl(x^{(k)},\delta/3\bigr).
\]
Let $V:=\operatorname{span}\{x^{(1)},\dots,x^{(m)}\}\subset X$ (so $\dim V\le m$), and let $P\colon X\to V$ be the orthogonal projector.

We claim that
\[
  \sup_{x\in K}\,\|x-Px\|<\delta/3.
\]
Indeed, for each $x\in K$ choose $k$ such that $\|x-x^{(k)}\|<\delta/3$. Then $x^{(k)}\in V$, hence the distance from $x$ to $V$ satisfies
\[
  \operatorname{dist}(x,V)\le \|x-x^{(k)}\|<\delta/3.
\]
But $\|x-Px\|=\operatorname{dist}(x,V)$ by the defining property of the orthogonal projection, which yields the desired estimate.

\textbf{Step 3 (passing to a function on $P(K)$).} Consider the compact set $K':=P(K)\subset V$.

From the construction $V=\operatorname{span}\{x^{(1)},\dots,x^{(m)}\}$ we have $P x^{(k)}=x^{(k)}$.
Moreover, Step 2 implies that for any $x\in K$ we have $\|x-Px\|<\delta/3$. Hence, since $K\subset\bigcup_{k=1}^m B\bigl(x^{(k)},\delta/3\bigr)$, for each $x\in K$ there exists $k$ such that
\[
  \|Px-x^{(k)}\|\le \|Px-x\|+\|x-x^{(k)}\|<\delta/3+\delta/3=2\delta/3.
\]
Therefore,
\[
  K'\subset \bigcup_{k=1}^m B_V\bigl(x^{(k)},2\delta/3\bigr),
\]
where $B_V(\cdot,\cdot)$ denotes a ball in the finite-dimensional space $V$.

Choose a continuous partition of unity $\{\psi_k\}_{k=1}^m$ on $K'$, subordinate to the cover $\{B_V(x^{(k)},2\delta/3)\}_{k=1}^m$ (existence of a partition of unity on a metric compact space; see\ \cite{munkres_topology}), i.e., $\psi_k\ge 0$, $\sum_k\psi_k\equiv 1$ on $K'$, and $\operatorname{supp}\psi_k\subset B_V(x^{(k)},2\delta/3)$.
Define a continuous function $F\in C(K')$ by
\[
  F(y):=\sum_{k=1}^m f\bigl(x^{(k)}\bigr)\,\psi_k(y),\qquad y\in K'.
\]

Estimate the approximation error of $f$ by $F\circ P$ on $K$.
Let $x\in K$ and $y=Px\in K'$. If $\psi_k(y)\neq 0$, then $y\in B_V(x^{(k)},2\delta/3)$, hence
\[
  \|x-x^{(k)}\|\le \|x-y\|+\|y-x^{(k)}\|<\delta/3+2\delta/3=\delta,
\]
and by the choice of $\delta$ we obtain $|f(x)-f(x^{(k)})|<\varepsilon/3$.
Then
\[
  |f(x)-F(Px)|
  =\Bigl|\sum_{k=1}^m \psi_k(Px)\,\bigl(f(x)-f(x^{(k)})\bigr)\Bigr|
  \le \sum_{k=1}^m \psi_k(Px)\,|f(x)-f(x^{(k)})|
  <\varepsilon/3.
\]
Consequently,
\[
  \sup_{x\in K}|f(x)-F(Px)|<\varepsilon/3.
\]

\textbf{Step 4 (approximation on a finite-dimensional compact set).} Identify $V$ with $\mathbb{R}^d$ ($d=\dim V$) via an orthonormal basis.

Consider the subalgebra $\mathcal{A}\subset C(K')$ (over $\mathbb{C}$) generated by the functions
\[
  y\longmapsto e^{i\langle t,y\rangle},\qquad t\in\mathbb{R}^d.
\]
That is, elements of $\mathcal{A}$ are finite linear combinations
\[
  y\longmapsto \sum_{j=1}^r c_j\,e^{i\langle t_j,y\rangle}.
\]
Clearly, $\mathcal{A}$ contains the constants and is closed under complex conjugation.
Moreover, $\mathcal{A}$ separates points of the compact set $K'$: if $y\neq y'$, then letting $u:=y-y'\neq 0$ and choosing $t:=u$, we get
\(
\langle t,y\rangle-\langle t,y'\rangle=\|u\|^2\neq 0
\), hence $e^{i\langle t,y\rangle}\neq e^{i\langle t,y'\rangle}$.

By the (complex) Stone--Weierstrass theorem \cite{rudin_rca}, the algebra $\mathcal{A}$ is dense in $C(K')$.
Therefore, there exists $T\in\mathcal{A}$ such that
\[
  \sup_{y\in K'}|F(y)-T(y)|<\varepsilon/3.
\]
Since $F$ is real-valued, replace $T$ by
\[
  \widetilde T:=\operatorname{Re} T=\tfrac12\bigl(T+\overline T\bigr)\in\mathcal{A}.
\]
Then for all $y\in K'$ we have
\(
|F(y)-\widetilde T(y)|\le |F(y)-T(y)|
\), and hence
\(
\sup_{y\in K'}|F(y)-\widetilde T(y)|<\varepsilon/3
\).
Renaming $\widetilde T$ back to $T$, we obtain a real-valued function that can be written as a finite sum of sines and cosines:
\[
  T(y)=\sum_{j=1}^r \alpha_j\cos\langle t_j,y\rangle+\sum_{j=1}^r \beta_j\sin\langle t_j,y\rangle.
\]
Each term has the form $\zeta(\langle t_j,y\rangle)$ with a continuous $\zeta\colon\mathbb{R}\to\mathbb{R}$ (namely, $\zeta(s)=\alpha_j\cos s$ or $\zeta(s)=\beta_j\sin s$), therefore
\[
  T(y)=\sum_{j=1}^r \zeta_j\bigl(\langle t_j,y\rangle\bigr).
\]

\textbf{Step 5 (return to $K\subset H^n$).} Define $g(x):=T(Px)$ for $x\in K$. Then
\[
  \sup_{x\in K}|f(x)-g(x)|
  \le
  \sup_{x\in K}|f(x)-F(Px)|+\sup_{x\in K}|F(Px)-T(Px)|
  <\varepsilon/3+\varepsilon/3=2\varepsilon/3<\varepsilon.
\]

It remains to bring $g$ to the required form.
Note that any continuous linear functional $L\in (H^n)^*$ has the form \cite{kreyszig_fa}
\[
  L(x_1,\dots,x_n)=\sum_{i=1}^n \varphi_i(x_i)
\]
for some $\varphi_i\in H^*$ (i.e., $(H^n)^*\cong (H^*)^n$). In particular, the functional
\(
L_j(x):=\langle t_j,Px\rangle
\)
is linear and continuous on $X=H^n$; therefore there exist $\varphi_{j1},\dots,\varphi_{jn}\in H^*$ such that
\[
  \langle t_j,P(x_1,\dots,x_n)\rangle=\sum_{i=1}^n \varphi_{ji}(x_i).
\]
Hence,
\[
  g(x_1,\dots,x_n)=\sum_{j=1}^r \zeta_j\Bigl(\sum_{i=1}^n \varphi_{ji}(x_i)\Bigr),
\]
\end{proof}

\section{Quantitative corollaries (metric entropy view)}
While the Stone--Weierstrass step is qualitative, the earlier compactness reduction already yields explicit finite-dimensional structure. We record a simple corollary in terms of covering numbers (metric entropy) \cite{kolmogorov_tikhomirov_entropy}.

\textbf{Corollary (finite-dimensional proxy via a net).} Fix $\varepsilon>0$ and let $\delta>0$ be chosen as in Step~1. Let $\mathcal{N}(K,\delta/3)$ denote the minimum size of a $\delta/3$-net of $K$ in the product norm on $H^n$. Then there exists a subspace $V\subset H^n$ with
\[
  \dim V\le \mathcal{N}(K,\delta/3)
\]
and an orthogonal projector $P\colon H^n\to V$ such that
\[
  \sup_{x\in K}\|x-Px\|<\delta/3
\]
and consequently
\[
  \sup_{x\in K}|f(x)-f(Px)|<\varepsilon/3.
\]

In particular, any quantitative strengthening of Step~4 on the compact set $P(K)\subset V\cong \mathbb{R}^{\dim V}$ immediately translates into quantitative approximation guarantees on $K$.

\section{Applications}
We include a small set of canonical examples whose main purpose is to illustrate the continuous linear measurements + scalar nonlinearities structure.

\subsection{Integral-type functionals on $L^2$}
Let $H=L^2(0,1)$ and $n=2$. Consider a compact set $K\subset H^2$ and a continuous functional
\[
  f(u,v)=\int_0^1 \Phi\bigl(u(x),v(x)\bigr)\,dx,
\]
where $\Phi\colon\mathbb{R}^2\to\mathbb{R}$ is continuous.

Then for any $\varepsilon>0$ there exist $r\in\mathbb{N}$, elements $h_{j1},h_{j2}\in L^2(0,1)$, and continuous functions $\zeta_j\colon\mathbb{R}\to\mathbb{R}$ such that
\[
  \sup_{(u,v)\in K}\left|f(u,v)-\sum_{j=1}^r \zeta_j\Bigl(\langle u,h_{j1}\rangle+\langle v,h_{j2}\rangle\Bigr)\right|<\varepsilon,
\]
where $\langle u,h\rangle:=\int_0^1 u(x)h(x)\,dx$. This can be read as finite sensing + scalar nonlinear post-processing of the inputs.

\subsection{Tomography and other linear inverse problems}
In tomography (CT, PET, etc.), the data are (approximately) finitely many linear measurements of an unknown image $u$, e.g., samples of a discretized Radon transform \cite{natterer_ct}. Learned reconstruction methods can then be viewed as nonlinear maps from measurements to the reconstructed object \cite{adler_oktem_learned}; related perspectives also appear in compressed sensing \cite{candes_cs}. The theorem provides an existence statement that, on compact sets of admissible images, any continuous scalar functional of the measured objects can be uniformly approximated by measurements + scalar nonlinearities.

\subsection{Relation to Kolmogorov--Arnold Networks (KAN)}
The formula
\[
  g(x_1,\dots,x_n)=\sum_{j=1}^r \zeta_j\Bigl(\sum_{i=1}^n \varphi_{ji}(x_i)\Bigr)
\]
can be read as a two-layer network with one hidden layer: on the edge from input $x_i$ to hidden neuron $j$ there is a one-dimensional function $x_i\mapsto \varphi_{ji}(x_i)$; then the hidden neuron applies a one-dimensional nonlinearity $\zeta_j$, after which one sums over $j$.

In terms of modern KANs \cite{liu_kan}, where learnable one-dimensional functions live on edges, the statement says that such a model class is dense in $C(K)$ provided the input edge functions are chosen to be linear and continuous (measurements of the input signals), while the nonlinearities $\zeta_j$ remain arbitrary continuous functions.

\paragraph{Further connections (brief).}
The same structural motif appears in transformers \cite{vaswani_transformer}, operator learning \cite{kovachki_neuraloperator,li_fno,lu_deeponet}, PINNs \cite{raissi_pinn}, and implicit neural representations \cite{sitzmann_siren,mildenhall_nerf}. We keep this discussion brief, as our focus here is the approximation theorem on compact subsets of $H^n$.

\paragraph{Remark on the literature.}
Related statements (on the density of cylindrical functions built from finitely many continuous linear functionals, as well as representations via sums of one-dimensional nonlinearities of linear forms) appear in the approximation and neural-network approximation literature; see, e.g.,\ \cite{pinkus_ridge,cybenko_ua,hornik_ua}.

\section{Conclusion}
We proved a universal approximation result for continuous functionals on compact sets $K\subset H^n$: every $f\in C(K)$ can be uniformly approximated by expressions built from finitely many continuous linear measurements and scalar nonlinearities,
\[
  g(x_1,\dots,x_n)=\sum_{j=1}^r \zeta_j\Bigl(\sum_{i=1}^n \varphi_{ji}(x_i)\Bigr),
\]
with $\varphi_{ji}\in H^*$ and continuous $\zeta_j\colon\mathbb{R}\to\mathbb{R}$. This directly matches the paper's theme: universal approximation on compact subsets of $H^n$ via linear measurements and scalar nonlinearities.

At a high level, the proof follows a two-stage route. Compactness yields a reduction to a finite-dimensional proxy via projection, and then classical approximation in finite dimensions (Stone--Weierstrass) produces ridge-type approximants. The intermediate finite-dimensional proxy step highlights an explicit connection to covering numbers/metric entropy \cite{kolmogorov_tikhomirov_entropy}.

Several directions remain open. A natural next step is to derive quantitative bounds (rates in $\varepsilon$) under additional regularity assumptions on $f$ and geometric assumptions on $K$, and to study how structural constraints on the measurements $\varphi_{ji}$ (e.g., locality, convolutional structure, band-limitation) affect approximation and learnability.

In Appendix~\ref{app:operator-valued} we record a Banach-valued (operator-output) extension of the theorem via finite-rank expansions, and in Appendices~\ref{app:cv}--\ref{app:ai} we sketch how the resulting linear measurement + scalar nonlinearity + linear synthesis template aligns with common models in imaging and AI.

\bibliography{references.bib}

\appendix

\section{Operator-valued extension}
\label{app:operator-valued}

Let $Y$ be a Banach space.

\begin{theorem}[Finite-rank approximation of $Y$-valued maps]
Let $K\subset H^n$ be compact, let $Y$ be a Banach space, and let $f\colon K\to Y$ be continuous.
Then for every $\varepsilon>0$ there exist $r\in\mathbb{N}$, vectors $y_1,\dots,y_r\in Y$, continuous linear functionals $\varphi_{ji}\in H^*$, and continuous scalar functions $\zeta_j\colon\mathbb{R}\to\mathbb{R}$ such that the map
\[
  g(x_1,\dots,x_n):=\sum_{j=1}^r y_j\,\zeta_j\Bigl(\sum_{i=1}^n \varphi_{ji}(x_i)\Bigr)
\]
satisfies
\[
  \sup_{x\in K}\,\|f(x)-g(x)\|_Y<\varepsilon.
\]
In particular, $g(K)$ is contained in a finite-dimensional subspace of $Y$.
\end{theorem}

\begin{proof}
Since $K$ is compact and $f$ is continuous, the image $f(K)\subset Y$ is compact.
Fix $\varepsilon>0$.

\textbf{Step 1 (finite-dimensional approximation of the range).}
Choose finitely many points $y^{(1)},\dots,y^{(m)}\in f(K)$ such that
\[
  f(K)\subset \bigcup_{k=1}^m B_Y\bigl(y^{(k)},\varepsilon/3\bigr).
\]
Let $W:=\operatorname{span}\{y^{(1)},\dots,y^{(m)}\}\subset Y$ (so $\dim W\le m$).
Let $q\colon W\to\mathbb{R}^d$ be any linear isomorphism (with $d=\dim W$), and let $i\colon W\hookrightarrow Y$ denote the inclusion.

\textbf{Step 2 (reduce to a $\mathbb{R}^d$-valued map).}
Consider the open cover of the compact set $f(K)$ by balls $\{B_Y(y^{(k)},\varepsilon/3)\}_{k=1}^m$.
Choose a continuous partition of unity $\{\psi_k\}_{k=1}^m$ on $f(K)$ subordinate to this cover.
Define the continuous map $\widetilde P\colon f(K)\to W$ by
\[
  \widetilde P(z):=\sum_{k=1}^m y^{(k)}\,\psi_k(z).
\]
If $\psi_k(z)\neq 0$, then $z\in B_Y(y^{(k)},\varepsilon/3)$, hence $\|z-y^{(k)}\|_Y<\varepsilon/3$.
Therefore, for all $z\in f(K)$,
\[
  \|z-\widetilde P(z)\|_Y
  =\left\|\sum_{k=1}^m \psi_k(z)\,(z-y^{(k)})\right\|_Y
  \le \sum_{k=1}^m \psi_k(z)\,\|z-y^{(k)}\|_Y
  <\varepsilon/3.
\]
Now define $F\colon K\to\mathbb{R}^d$ by $F(x):=q\bigl(\widetilde P(f(x))\bigr)$.
Then $\|f(x)-i\bigl(\widetilde P(f(x))\bigr)\|_Y<\varepsilon/3$ for all $x\in K$.

\textbf{Step 3 (approximate each coordinate by Theorem~1).}
Write $F=(F_1,\dots,F_d)$.
Applying Theorem~1 to each scalar map $F_\ell\colon K\to\mathbb{R}$ and accuracy $\varepsilon/(3\,\|q^{-1}\|\,\|i\|\,d)$ yields functions
\[
  G_\ell(x)=\sum_{j=1}^{r} a_{\ell j}\,\zeta_j\Bigl(\sum_{i=1}^n \varphi_{ji}(x_i)\Bigr)
\]
with $a_{\ell j}\in\mathbb{R}$ such that $\sup_{x\in K}|F_\ell(x)-G_\ell(x)|<\varepsilon/(3\,\|q^{-1}\|\,\|i\|\,d)$.
Let $G:=(G_1,\dots,G_d)$.

\textbf{Step 4 (lift back to $Y$).}
Define
\[
  g(x):=i\bigl(q^{-1}(G(x))\bigr)\in Y.
\]
Then $g$ has the required form with vectors $y_j:=i\bigl(q^{-1}((a_{1j},\dots,a_{dj}))\bigr)\in Y$.
Moreover, for all $x\in K$,
\[
  \|i\bigl(\widetilde P(f(x))\bigr)-g(x)\|_Y
  \le \|i\|\,\|q^{-1}\|\,\|F(x)-G(x)\|_1
  \le \|i\|\,\|q^{-1}\|\sum_{\ell=1}^d |F_\ell(x)-G_\ell(x)|
  <\varepsilon/3.
\]
Combining with Step~2 gives $\sup_{x\in K}\|f(x)-g(x)\|_Y<\varepsilon$.
\end{proof}

\section{Application of Theorem~2 to computer vision and imaging}
\label{app:cv}

As a simple illustration, take $H=L^2(\Omega)$ with $\Omega\subset\mathbb{R}^2$ the image domain, and let $K\subset H$ be a compact set of admissible images (e.g., a compact model of ``natural images'' after normalization).
Let $Y=L^2(\Omega)$ represent an output image/field (denoised image, reconstruction, depth/flow field, segmentation logits, etc.), and let $f\colon K\to Y$ be a continuous image-to-image map.

Theorem~2 implies that for every $\varepsilon>0$ there exist finitely many output ``templates'' $y_1,\dots,y_r\in Y$, finitely many linear measurements $\varphi_j\in H^*$ (e.g., filter responses, Fourier/Radon samples, wavelet coefficients), and continuous scalar nonlinearities $\zeta_j$ such that \cite{natterer_ct,adler_oktem_learned,mallat_wavelet}
\[
  \sup_{x\in K}\left\|f(x)-\sum_{j=1}^r y_j\,\zeta_j\bigl(\varphi_j(x)\bigr)\right\|_{L^2(\Omega)}<\varepsilon.
\]
In imaging language, this is a universal-approximation justification (on compact data regimes) for pipelines that first compute a finite set of linear features of the input image and then synthesize the output as a linear combination of finitely many basis images with nonlinear (scalar) coefficients.

\section{Application of Theorem~2 to artificial intelligence}
\label{app:ai}

Theorem~2 also applies broadly to AI settings where the inputs and outputs are naturally modeled as elements of function spaces.
For example, in operator learning one aims to learn maps between function spaces (e.g., mapping coefficients/forcing terms to PDE solutions), which can be viewed as continuous maps $f\colon K\to Y$ on compact data regimes \cite{kovachki_neuraloperator,li_fno,lu_deeponet}.
Theorem~2 gives an existence guarantee that such maps can be approximated uniformly on $K$ by finite-rank models built from finitely many linear measurements of the input signals and scalar nonlinearities.

More concretely, if $H=L^2(\Omega)$ and $K\subset H$ is a compact set of admissible inputs (fields, signals, contexts), then any continuous $f\colon K\to Y$ can be approximated by
\[
  g(x)=\sum_{j=1}^r y_j\,\zeta_j\bigl(\varphi_j(x)\bigr),
\]
which matches the common ``encode--nonlinearity--decode'' template: linear encoders $\varphi_j$ (features), scalar nonlinear mixing $\zeta_j$, and a linear decoder spanned by output atoms $y_j$.
This viewpoint helps unify a range of architectures (e.g., linear-feature front ends followed by MLPs, low-rank decoders, and neural-operator models) as instances of the same universal approximation principle on compact sets.

%\appendix

% \crefalias{section}{appendix} % uncomment if you are using cleveref

%\section{My Proof of Theorem 1}

%This is a boring technical proof.

%\section{My Proof of Theorem 2}

%This is a complete version of a proof sketched in the main text.

\end{document}